\documentclass[letterpaper,journal,onecolumn]{IEEEtran}

\usepackage[utf8]{inputenc} 
\usepackage[T1]{fontenc}    
\usepackage{hyperref}       
\usepackage{url}            
\usepackage{booktabs}       
\usepackage{amsfonts}
\usepackage{amsthm}
\usepackage{amsmath}
\usepackage{nicefrac}       
\usepackage{microtype}      
\usepackage{xcolor}         
\usepackage{algorithm}
\usepackage{algorithmic}
\newtheorem{theorem}{Theorem}

\newtheorem{lemma}{Lemma}

\usepackage{dsfont}
\usepackage{bbm,comment}
\usepackage{placeins} 
\usepackage{graphicx}
\renewcommand{\l}{\ell}
\renewcommand{\th}{{\mathrm th}}

\title{Threshold-Based Optimal Arm Selection in Monotonic Bandits: Regret Lower Bounds and Algorithms}


\author{%
Chanakya Varude$^{1}$ \quad Jay Chaudhary$^{1}$ \quad Siddharth Kaushik$^1$ \quad Prasanna Chaporkar$^{1}$\\
$^1$Department of Electrical Engineering, Indian Institute of Technology, Bombay\\
\texttt{\{210070092,21007022,210070086\}@iitb.ac.in}\\
\texttt{\{chaporkar\}@ee.iitb.ac.in}
}

\begin{document}

\maketitle

\begin{abstract}
In multi-armed bandit problems, the typical goal is to identify the arm with the highest reward. This paper explores a threshold-based bandit problem, aiming to select an arm based on its relation to a prescribed threshold \(\tau \). We study variants where the optimal arm is the first above \(\tau\), the \(k^{th}\) arm above or below it, or the closest to it, under a monotonic structure of arm means. We derive asymptotic regret lower bounds, showing dependence only on arms adjacent to \(\tau\). Motivated by applications in communication networks (CQI allocation), clinical dosing, energy management, recommendation systems, and more. We propose algorithms with optimality validated through Monte Carlo simulations. Our work extends classical bandit theory with threshold constraints for efficient decision-making.
\end{abstract}

\section{Introduction}
Many real-world decision-making problems involve selecting an option from ordered alternatives with increasing (or decreasing) but unknown payoffs. For instance, clinical dose-finding seeks the lowest dose meeting an efficacy threshold, while pricing and resource allocation aim to identify the least costly configuration exceeding a performance target. These are modeled as stochastic bandit problems with ordered arms, focusing on identifying an arm based on its position relative to a threshold, rather than the best arm or full classification.

In this paper we study a family of structured stochastic bandit problems in which the means of the arms are known to be monotonically increasing (or decreasing). Our focus is on threshold-based \textit{identification tasks}. Given a threshold $\tau \in [0, 1]$, our goal is to identify an optimal arm defined as the first arm whose mean exceeds $\tau$, the $k$-th arm above (or below) $\tau$, or the one closest in value to it. 

These threshold-based identification tasks address critical real-world applications requiring fast convergence in time-sensitive, costly systems like wireless communication (estimating channel quality thresholds to optimize bandwidth and power allocation under noisy conditions), medical testing (assessing biomarker thresholds to ensure safe, effective dosing), energy management (identifying optimal power settings to meet efficiency targets) and recommendation systems (evaluating preference thresholds to deliver timely, relevant suggestions). Slow convergence incurs significant costs, such as inefficient resource use, suboptimal treatments, excessive energy consumption, or delayed user engagement. Our algorithms minimize cumulative regret, reducing losses from suboptimal arm selections. Unlike pure exploration approaches, which prioritize final identification accuracy and are less effective in dynamic, resource-constrained settings, our focus ensures ongoing efficiency for real-time decision-making under uncertainty.

While traditional multi-armed bandit (MAB) literature has largely emphasized cumulative reward maximization, with many studies focusing on selecting the arm with the maximum mean reward \cite{auer2002using, audibert2010best, bubeck2009pure, gabillon2012best, karnin2013almost, jamieson2014lilucb, garivier2016optimal, chen2017instance}, the exploration of \textit{pure identification problems}—such as best-arm identification or thresholding—has become an active area of research. In particular, the thresholding bandit problem (TBP), where the goal is to classify arms relative to a fixed threshold, has gained significant attention. This shift in focus has led to extensive studies in both the fixed-confidence and fixed-budget settings, addressing applications where identifying arms relative to a threshold is critical.

\subsection{Related Work}
Early work by \cite{chen2014combinatorial} introduced the LUCB algorithm in the context of combinatorial pure exploration (CPE), a general framework in which the learner seeks to identify a subset of arms satisfying some combinatorial constraint (e.g., means above $\tau$). Their work provided PAC-style guarantees in the fixed-confidence regime, but did not exploit any structural assumptions about the arm means, and hence could be inefficient in practice when such structure exists.

Work based on structured bandits done in \cite{DBLP:journals/corr/CombesP14a} investigate stochastic multi-armed bandit problems where the expected reward function is unimodal over a partially ordered set of arms. They derive asymptotic lower bounds for regret and propose the OSUB algorithm, which achieves optimal regret performance independent of the number of arms.

In the fixed-budget setting, ~\cite{locatelli2016optimal} proposed the APT algorithm and established fundamental lower and upper bounds on the probability of misclassification. They showed that the optimal error probability for TBP decays as $\exp(-T / H)$, where $H = \sum_k \Delta_k^{-2}$ is a complexity term based on the gaps $\Delta_k = |\mu_k - \tau|$. Although asymptotically optimal, APT is tailored to the unstructured case, thereby failing to exploit shared structure or ordering between arms.

Building on this, Mukherjee et al.~\cite{mukherjee2017thresholding} introduced the Augmented-UCB (AugUCB) algorithm, which enhances arm elimination strategies by incorporating both empirical means and variances. AugUCB improves performance over earlier methods, but remains focused on full classification under a fixed sampling budget, rather than specific identification objectives.

Later developments by ~\cite{agarwal2017learning} and ~\cite{chowdhury2017kernelized} generalized the TBP to contextual and nonparametric domains, studying level-set estimation with linear or kernelized function classes. While their work inspired extensions of TBP to more general input spaces, they operate in continuous domains with smoothness assumptions and are orthogonal to our setting.

Closely related is the thresholding bandit study by ~\cite{garivier2018thresholding}, analyzing sample complexity for identifying the arm closest to threshold \( S \) under monotonicity with optimal algorithms for Gaussian distributions, but focusing on pure exploration. Similarly, ~\cite{JMLR:v22:19-228} adapt Thompson Sampling for dose-finding, offering finite-time bounds on sub-optimal allocations and not specifically exploiting monotonic structure, contrasting with our regret minimization focus. Closer still is the structured thresholding study by ~\cite{cheshire2021structured}, addressing the TBP under monotonicity and concavity constraints, proposing PD-MTB and PD-CTB algorithms. These use binary search with corrections to achieve problem-dependent rates but focus on pure exploration for full classification, neglecting rapid convergence essential for resource-sensitive applications.

\subsection{Our Contributions}
In this paper, we study threshold-based identification in structured bandits with strict monotonic structure \(\mu_1 < \mu_2 < \cdots < \mu_K\) ( or \(\mu_1 > \mu_2 > \cdots > \mu_K\)). We propose algorithms for three tasks:
\begin{enumerate}
    \item \textbf{Threshold-crossing identification:} Identify the smallest index $k^*$ such that $\mu_{k^*} \geq \tau$.
    \item \textbf{$\l^\th$ above(below) threshold identification:} Identify the $\l^\th$ arm above(below) $\tau$.
    \item \textbf{Threshold proximity identification:} Identify the arm closest to $\tau$. 
\end{enumerate}
Each task captures a distinct operational requirement in structured settings. For example, {\it threshold-crossing} models safety-critical applications, where we need the minimally sufficient dosage or investment level. {\it Proximity identification} captures scenarios like optimal pricing or resource tuning. The $\l^\th$ {\it arm identification} allows access to arbitrary quantiles with respect to the threshold.

For each of the three problems, we provide: \textbf{1)} the asymptotic regret lower bound for any uniformly good algorithm, and  \textbf{2)} efficient algorithm that empirically shown to approach the lower bound.

In contrast to prior work on full classification, our results address identification problems, offering a new perspective on structural bandits for resource-sensitive applications. Our algorithms are simple, practical, and avoid the complexity of full classification while achieving theoretical guarantees.
\section{Model and Objectives}
\noindent
We study a stochastic multi-armed bandit problem with \( K \geq 2 \) arms, indexed by the set \( S := \{1, \ldots, K\} \). 
Unlike the classical objective of maximizing expected rewards, we focus on threshold-based objectives, aiming to identify arms whose mean rewards satisfy specific conditions relative to a known threshold \( \tau \in (0, 1) \). 
Time proceeds in discrete rounds \( t = 1, 2, \ldots \). At each round, a policy \( \pi \in \Pi \) selects an arm \( k^\pi(t) \in S \) based on the history of actions and observations up to time \( t-1 \). Formally, let \( \mathcal{F}_{t-1}^\pi \) be the \( \sigma \)-algebra generated by \( \{(k^\pi(s), X_{k^\pi(s)}(s)) : 1 \leq s \leq t-1\} \), where \( X_{k^\pi(s)}(s) \) is the reward observed at time \( s \). The choice \( k^\pi(t) \) is \( \mathcal{F}_{t-1}^\pi \)-measurable. An instance \( \mathcal{I} \) is defined by a parameter vector \( \theta = (\theta_1, \ldots, \theta_K) \), where each arm \( k \in S \) is associated with a distribution \( D_k = \nu(\theta_k) \), supported on \( [0, 1] \), with unknown mean \( \mu_k = \mu(\theta_k) = \mathbb{E}_{X \sim D_k}[X] \). The rewards \( \{ X_k(t) \}_{t \geq 1} \) for arm \( k \) are i.i.d.\ draws from \( D_k \), and rewards are independent across arms. Let \( \mu = (\mu_1, \ldots, \mu_K) \in [0, 1]^K \) denote the vector of mean rewards.

The optimal arm \( k^* \) is defined by one of the following threshold-based objectives relative to \( \tau \in (0, 1) \): 

(i) \emph{Threshold-crossing identification:} given by \( k^*_{(1)} = \min \{ k \in S : \mu_k \geq \tau \} \), 

(ii) \emph{\( \l^\th \) above (below) threshold identification:} given by \( k^*_{(2,a)} = k^*_{(1)} + \l-1 \) (\( k^*_{(2,b)} = k^*_{(1)}-\l \)),

(iii) \emph{Threshold proximity identification:} Given by  \( k^*_{\text{cl}} = \arg\min_{k \in S} |\mu_k - \tau| \).

For each objective, we only consider instances that have a valid optimal arm. Next, we define the monotonic structure assumption and the expected regret on which the performance shall be measured.\\
\vspace{-10 pt}

\subsection{Monotonic Structure} 
We assume the arms follow a strict monotonic structure, where their mean rewards form an ordered sequence, either increasing (\( \mu_1 < \mu_2 < \cdots < \mu_K \)) or decreasing (\( \mu_1 > \mu_2 > \cdots > \mu_K \)). We focus on the increasing case, with the learner aware of this ordering but not the exact \( \mu_k \) values. A decreasing sequence can be mapped to an increasing one by reindexing arms as \( k \to K+1-k \), so our results extend naturally. 

\subsection{Regret Definition}
The performance of an algorithm \( \pi \in \Pi \) is characterized by its regret up to time \( T \), where \( T \) is typically large. 
Let \( t^{\pi}_k(T) = \sum_{1\leq n\leq T} \mathds{1}\{ k^{\pi}(n) = k \} \).
The regret \( R^{\pi}(T) \) of an algorithm \( \pi \) is given by:
\begin{equation*}
R^{\pi}(T) = \sum_{k=1}^{K} |\mu_{k^*} - \mu_k| \mathbb{E}[t^{\pi}_k(T)],
\end{equation*}
\noindent
where $k^*$ is the optimal arm defined according to the threshold-based objective. This cumulative regret measures the expected loss from selecting suboptimal arms, which is critical in applications such as wireless communication, clinical dosing, energy management, or recommendation systems, where each suboptimal choice incurs significant costs, including inefficient resource use or delayed outcomes. Unlike pure exploration approaches (e.g., fixed-budget or fixed-confidence settings), which focus on minimizing identification errors at the end of a fixed period, cumulative regret ensures rapid convergence to the optimal arm, aligning with the time and cost-sensitive nature of our tasks.
\section{Asymptotic Lower Bound on Regret}
In this section, we find asymptotic lower bounds on the regret for {\it uniformly good algorithms} defined as policies \( \pi \in \Pi \) where regret \( R^\pi(T) \) grows slower than \( T^a \) for any \( a > 0 \) and all feasible instances.
These lower bounds show the optimal performance an algorithm can achieve, giving us a way to check the optimality of our algorithms.
We assume the reward distributions belong to a parametrized family. Specifically, we define a set of distributions \( D = \{ \nu(\theta) \}_{\theta \in [0,1]} \), where each \( \nu(\theta) \) has mean \( \mu(\theta) \). Each \( \nu(\theta) \) has a density \( p(x, \theta) \) with respect to a positive measure \( m \) on \( \mathbb{R} \), i.e \(\nu(dx,\theta) = p(x,\theta)m(dx)\). The Kullback-Leibler (KL) divergence between \( \nu(\theta) \) and \( \nu(\theta') \) is given as:
\[
D_{KL}(\theta ||\theta') = \int_{\mathbb{R}}p(x, \theta) \log \left( \frac{p(x, \theta)}{p(x, \theta')} \right)  \, m(dx).
\]
Similarly, the divergence between a distribution \(\nu(\theta)\) and the threshold $\tau \in (0,1)$ can be computed by taking $\tau$ as a parameter, using the same form. We define the parameter space \( \Theta_{\tau} \) as the set of all vectors \( \theta = (\theta_1, \dots, \theta_K) \in [0,1]^K \), where the corresponding means \( \mu_k = \mu(\theta_k) \) satisfy the strict monotonic condition \( \mu_1 < \mu_2 < \dots < \mu_K \), \(\mu_1 < \tau\) and \(\mu_K\geq\tau\) (to ensure existence of at least one arm below and above $\tau$). Thus for each arm \( k \), we denote its parameter as \( \theta_k \), so \( \nu_k = \nu(\theta_k) \). We define the minimal divergence number between two distributions \(\nu(\theta_1) \text{ and } \nu(\theta_2) \text{ as}\):
\[
I_{\min}(\theta_1, \theta_2) = \inf_{\theta' : \mu(\theta') \geq \theta_2} D_{KL}(\theta_1|| \theta')
\]

We now derive the tightest possible regret lower bound as the time horizon \( T \) approaches infinity for all threshold variants. Our analysis centers on Bernoulli distributions, where $\mu(\theta) = \theta, $ and $I_{min}(\theta_1,\theta_2) = I(\theta_1,\theta_2) = \theta_1\ln(\frac{\theta_1}{\theta_2}) + (1-\theta_1)\ln(\frac{1-\theta_1}{1-\theta_2})$.
\noindent
\begin{theorem}[Lower bound for Threshold-crossing Identification]\label{thm:1}
    Let $\pi \in \Pi$ be a uniformly good algorithm for the threshold-crossing identification problem.  Then, for any $\theta \in \Theta_{\tau}$,
    \begin{equation*}
        \liminf_{T \to \infty} \frac{R^{\pi}(T)}{\log T} 
        \geq \frac{|\mu_{k^*} - \mu_{k^\prime}|}{I_{\min}(\theta_{k^\prime}, \tau)},
    \end{equation*}
    where $k^* = \arg\min \{ k \mid \mu_k \geq \tau \}$ is the optimal arm, and $k^\prime = k^* - 1$.
\end{theorem}
\begin{proof}
Consider any instance $\theta \in \Theta_\tau$. Let $k^*$ denote the optimal arm for this $\theta$, i.e. $\tau \in (\mu_{k^*-1}, \mu_{k^*}]$.

Now, we define the set $B(\theta)$ of bad parameters $\lambda$, such that $k^*$ is not optimal in $\lambda$ and is statistically indistinguishable from $\theta$. Specifically,
\begin{equation*}
B(\theta) = \left\{ \lambda \in \Theta_{\tau} : \mu_{k^*}(\lambda) = \mu_{k^*}(\theta) \text{ and } \exists\, k \neq k^* \text{ s.t. } \tau \in (\mu_{k-1}(\lambda),\mu_k(\lambda)] \right\}.
\end{equation*}

The optimization problem that yields the lower bound, as a consequence of results in (\cite{Graves1997AsymptoticallyEA}), is given by:
\begin{equation}\label{eq:gl_opt_1}
C(\theta) = \inf \left\{ \sum_{k \neq k^*} C_k |\mu_{k^*} - \mu_k| : C_k \geq 0, \right.
\left. \inf_{\lambda \in B(\theta)} \sum_{k \neq k^*} C_k I(\theta_k, \lambda) \geq 1 \right\}.
\end{equation}

For each suboptimal arm $i$, lets define \(B_i(\theta) = \{ \lambda \in B(\theta) : \tau \in (\mu_{i-1}(\lambda),\mu_i(\lambda)] \}.\) Thus, $B_i(\theta)$ denotes collection of bad parameters for which arm $i$ (instead of $k^*$)  is optimal. 
Note that because of monotonic structure, $B_i(\theta) = \emptyset, \  \forall i > k^*$. Moreover, $\{B_i(\theta)\}_{i<k^*}$ is a partition of $B(\theta)$.
Thus, optimization (\ref{eq:gl_opt_1}) is equivalent to minimizing $\sum_{k \neq k^*} C_k |\mu_{k^*} - \mu_k|$ subject to the following constraints:
\begin{equation}\label{eq:constraint_1}
C_k \geq 0, \quad \inf_{\lambda \in B_i(\theta)} \sum_{k \neq k^*} C_k I(\theta_k, \lambda_k) \geq 1, \quad \forall i < k^*.
\end{equation}
For a fixed \( i < k^*\), consider $\lambda \in B_i(\theta)$:
\begin{equation}\label{eq:lambda_1}
\lambda = \{ \theta_1, \dots, \theta_{i-1}, b^i, b^{i+1}, \dots,b^{k^*-1} ,\theta_{k^*}, \dots, \theta_K \}, \text{where } \mu(b^i) \geq \tau  \text{ and }  \mu(\theta_{i-1}) < \tau.
\end{equation}
Using (\ref{eq:lambda_1}) we can rewrite (\ref{eq:constraint_1}) as
\begin{align}\label{eq:constr_1}
\inf_{b : \mu(b) \geq \tau} C_i I(\theta_i, b^i) + \cdots +
C_{k^*-1} I(\theta_{k^*-1}, b^{k^*-1}) \geq 1, \ \forall i < k^*.
\end{align}
Note that the constraint for each $i<k^*$ in (\ref{eq:constr_1}) has a common term of arm $k^*-1$. Also note that since $\theta_{k^*-1} < \tau$, and $I(a,b+u)$ is monotone increasing in $u$ for $a<b$ and $u>0$,
    $\inf_{b^{k^*-1} \geq \tau} I(\theta_{k^*-1}, b^{k^*-1}) = I(\theta_{k^*-1}, \tau)$.
The optimal solution for (\ref{eq:gl_opt_1}) is:
\begin{equation}\label{eq:opt_sol_1}
C_{k^*-1} I(\theta_{k^*-1}, \tau) = 1 , \quad C_i = 0 \text{ for all }  i \not= k^* - 1.
\end{equation}
\noindent
Solving the optimization problem (\ref{eq:gl_opt_1}) using (\ref{eq:opt_sol_1}) yields the lower bound:
\[
\liminf_{T \to \infty} \frac{R^{\pi}(T)}{\log T} \geq C(\theta) = \frac{|\mu_{k^*} - \mu_{k^*-1}|}{I(\theta_{k^*-1}, \tau)}
\]
This proves the required.
\end{proof}
In classical multi-armed bandit problems, identifying the arm with maximum mean reward yields a logarithmic regret bound (\cite{LAI19854}), and so does this threshold-based variant selecting the arm with mean just above \(\tau\). \\
Given a monotonic structure, the tightest regret lower bound depends on distinguishing the suboptimal arm immediately below the threshold (i.e., \(k^* - 1\)) from \(\tau\). Due to monotonicity, this distinction determines whether the next arm (i.e., \(k^*\)) lies above or below \(\tau\). Thus, it only involves the arm below the threshold. We next derive a lower bound for the second threshold variant.

\begin{theorem}[Lower bound for $\l^\th$ arm above Threshold Identification]\label{thm:2}
 Let $\pi \in \Pi$ is a uniformly good algorithm for $\l^\th$ arm above threshold identification. Then, for any $\theta \in \Theta_{\tau}$, 
    \[
        \liminf_{T \to \infty} \frac{R^{\pi}(T)}{\log T} \geq \frac{|\mu_{k^*} - \mu_{k'}|}{I_{\min}(\theta_{k'}, \tau)}\mathds{1}\{k^* \neq K\}  + \frac{|\mu_{k^*} - \mu_{k'-1}|}{I_{\min}(\theta_{k'-1}, \tau)}  ,
    \]
where $k' = \arg\min \{ k\mid \mu_k \geq \tau \}$, and $k^* = k'+\l-1$ is the optimal arm.
\end{theorem}
\begin{proof}
    Following the same steps as in the proof of Theorem~\ref{thm:1}, consider any instance $\theta \in \Theta_\tau$. Let $k^*$ denote the optimal arm for this $\theta$, i.e. $\tau \in (\mu_{k^*-\l}, \mu_{k^*-\l+1}]$. Now, set of bad parameters is given by:
\begin{equation*}
B(\theta) = \left\{ \lambda \in \Theta_{\tau} : \mu_{k^*}(\lambda) = \mu_{k^*}(\theta) \text{ and } \exists\, k \neq k^* :   \mu_{k-\l+1}(\lambda) \ge \tau \text{ and } \mu_{k-\l}(\lambda) < \tau\right\}.
\end{equation*}
For each suboptimal arm $i$, define \(B_i(\theta) = \{ \lambda \in B(\theta) : \mu_{i-\l+1}(\lambda) \ge \tau \text{ and } \mu_{i-\l}(\lambda) < \tau  \},
\)
\noindent
also let $k' = \arg\min \{ k\mid \mu_k(\theta) \geq \tau \}$, 
solving the same optimization problem as in (\ref{eq:gl_opt_1}) and  (\ref{eq:constraint_1}), we first construct $\lambda$ from $B_i(\theta)$. 
It is to be noted that for a strict monotonic structure: $B_i(\theta) = \emptyset, \forall i < \l+1 $ or $i \geq k^*+\l$. For a fixed suboptimal arm $i$ such that  \(\min (k^*+\l-1,K) \geq i > k^* \). Let $d_1 = i-k^*-1 $ then:
\begin{equation}\label{eq:bad1_2}
\lambda = \{ \theta_1, \dots, \theta_{k'-1}, b^{k'}, b^{k'+1}, \dots,b^{k'+d_1} ,\theta_{k'+d_1+1}, \dots,\theta_i,\dots \theta_K \} \in B_i(\theta), 
\end{equation}
where \(\mu(b^{k'+d_1}) <\tau. \) Alternatively, for a fixed suboptimal arm $i$ such that $\l+1<i<k^*$. Let $d_2 = k^*-i$, then:
\begin{equation}\label{eq:bad2_2}
\lambda = \{ \theta_1, \dots,\theta_{k'-d_2-1},b^{k'-d_2},\dots, b^{k'-1}, \theta_{k'},\dots,\theta_i,\dots \theta_K \} \in B_i(\theta),   
\end{equation}
where $\mu(b^{k'-d_2}) \geq \tau$. Using (\ref{eq:bad1_2}) to solve the constraints (\ref{eq:constr_1}), we get 
\begin{align}\label{eq:constraintd1}
    \inf_{b :\mu(b^{k'+d_1}) <\tau } C_{k'} I(\theta_{k'}, b^{k'}) + \dots +
C_{k'+d_1} I(\theta_{k'+d_1}, b^{k'+d_1}) \geq 1,\forall \min(k^*+l-1,K)\geq i > k^*
\end{align}
Note that the constraint for each $\min(k^*+\l-1,K)\geq i>k^*$ in (\ref{eq:constraintd1}) has a common term of arm $k'$. Also note that since $\theta_{k'} \geq \tau$, and $I(a,b-u)$ is monotone increasing in $u$ for $a>b$ and $u>0$,
    $\inf_{b^{k'} < \tau} I(\theta_{k'}, b^{k'}) = I(\theta_{k'}, \tau)$.
The optimal solution for (\ref{eq:gl_opt_1}) is:
\begin{equation} \label{eq:sol1_2}
C_{k'} I(\theta_{k'}, \tau) = 1 , \quad C_i = 0 \quad \forall i \not\in \{k',k^*\}. 
\end{equation}
Using (\ref{eq:bad2_2}) to solve the constraints (\ref{eq:constr_1}), we get
\begin{align}\label{eq:constraintd2}
    \inf_{b :\mu(b^{k'-d_2}) \geq\tau } C_{k'-d_2} I(\theta_{k'-d_2}, b^{k'-d_2}) + \dots +
C_{k'-1} I(\theta_{k'-1}, b^{k'-1}) \geq 1,\forall k^*>i > \l+1.
\end{align}

Note that the constraint for each $\l+1<i<k^*$ in (\ref{eq:constraintd2}) has a common term of arm $k'-1$. Also note that since $\theta_{k'-1} < \tau$, and $I(a,b+u)$ is monotone increasing in $u$ for $a<b$ and $u>0$,
    $\inf_{b^{k'-1} \geq \tau} I(\theta_{k'-1}, b^{k'-1}) = I(\theta_{k'-1}, \tau)$.
The optimal solution for (\ref{eq:gl_opt_1}) is:
\begin{equation}\label{eq:sol2_2}
C_{k'-1} I(\theta_{k'-1}, \tau) = 1 , \quad C_i = 0 \quad \forall i \not\in\{ k' - 1,k^*\}.
\end{equation}
Hence, taking union of the solutions (\ref{eq:sol1_2}) and (\ref{eq:sol2_2}) $\forall i \neq k^*$, the optimization yields the asymptotic lower bound:
\[
\liminf_{T \to \infty} \frac{R^{\pi}(T)}{\log T} \ge C(\theta) = \frac{|\mu_{k^*} - \mu_{k'}|}{I_{\min}(\theta_{k'}, \tau)}\mathds{1}\{k^* \neq K\} + \frac{|\mu_{k^*} - \mu_{k'-1}|}{I_{\min}(\theta_{k'-1}, \tau)}.
\]
This proves the required.
\end{proof}
For variant (ii) defined in Section 1, the regret bound grows logarithmically, driven by the challenge of distinguishing the arms immediately above (\( k' \)) and below (\( k' - 1 \)) the threshold. These neighboring arms establish a reference point, enabling straightforward identification of the optimal arm as the \( \l^\th \) arm above \( \tau \). Notably, when \( \l = 1 \), this regret bound reduces to that of Theorem~\ref{thm:2}.\\

The theorems above apply not only to increasing sequences of arm means but also to decreasing ones, as shown by a straightforward parameter transformation detailed in the lemma that follows.

\begin{lemma}
 The regret bounds for identifying the arm just above the threshold $\tau \in (0,1)$ (Theorem~\ref{thm:1}) and the $\l^{\th}$ arm above $\tau$ (Theorem~\ref{thm:2}) directly extend to the settings of identifying the arm just below $\tau$ and the $\l^{\th}$ arm below $\tau$ under a strict monotone increasing structure, via problem inversion.
\end{lemma}
\begin{proof}
Define transformed rewards \( Y_k = 1 - X_k \), where \( X_k \sim \nu(\theta_k) \) is supported on \( [0, 1] \), so \( Y_k \) has mean \( \mu_k^Y = 1 - \mu_k \). Reindex arms as \( k' = K + 1 - k \), mapping original arm \( k = K + 1 - k' \) to arm \( k' \) with mean \( \mu_{k'}^Y = 1 - \mu_{K + 1 - k'} \). Since \( \mu_1 < \mu_2 < \dots< \mu_K \), we have \( \mu_1^Y > \mu_2^Y > \dots > \mu_K^Y \), preserving monotonicity. The arm just below \( \tau \), \( \max \{ k : \mu_k < \tau \} \), corresponds to the arm just above \( 1 - \tau \), \( \min \{ k' : \mu_{k'}^Y > 1 - \tau \} \), in the transformed problem. Similarly, the \( \l^\th \) arm below \( \tau \) maps to the \( \l^\th \) arm above \( 1 - \tau \). The regret depends on mean differences, which remain equivalent, and KL divergences \( I(\theta_k, \lambda_k) \) are preserved. Thus, the regret bounds from Theorems~\ref{thm:1} and~\ref{thm:2} apply to the below-threshold variants. Algorithms optimal for the above-threshold objectives, matching these bounds up to vanishing terms, remain optimal for the transformed problem.
\end{proof}
\begin{theorem}[Lower bound for Threshold Proximity Identification]\label{thm:3}
 Let $\pi \in \Pi$ is a uniformly good algorithm for threshold proximity identification problem. Then, for any $\theta \in \Theta_{\tau}$,  the regret bound is:
    \[
        \liminf_{T \to \infty} \frac{R^{\pi}(T)}{\log T} \geq \frac{|\mu_{k^*} - \mu_{k}|}{I_{\min}(\theta_{k}, 2\tau - \theta_{k^*})},
    \]
where $k^* = \arg\min \{ |\mu_k - \tau|\}$ is the optimal arm, and 
\(
k =
\begin{cases}
    k^*-1 & \text{if  } \mu_{k^*}\geq \tau \\
    k^*+1 & \text{otherwise}
\end{cases}
 \quad.\)
\end{theorem}
\begin{proof}
    Following the same steps as in the proof of Theorem~\ref{thm:1}, consider any instance $\theta \in \Theta_\tau$. Let $k^*$ denote the optimal arm for this $\theta$, i.e. $k^* = \arg\min_k (|\mu_{k}-\tau|)$. The set of bad parameters is:
\begin{equation*}
B(\theta) = \left\{ \lambda \in \Theta_{\tau} : \mu_{k^*}(\lambda) = \mu_{k^*}(\theta) \text{ and } \exists\, k \neq k^* : \right.
\left. |\mu_k(\lambda) - \tau| < |\mu_{k^*}(\lambda) - \tau| \right\}.
\end{equation*}
Without loss of generality, let us consider $ \mu_{k^*}\geq \tau$. Now, for each suboptimal arm $i$, define
\[
B_i(\theta) = \{ \lambda \in B(\theta) : i = \arg\min_j|\mu_j(\lambda) - \tau| \}.\]
\noindent
Solving the same optimization problem as in (\ref{eq:constr_1}), we first construct $\lambda$ from $B_i(\theta)$.
It is to be noted that for a monotonic structure: $B_i(\theta) = \emptyset \quad \forall i > k^*$, For a fixed \( i < k^* \), let:
\begin{equation}\label{eq:bad_para_3}
\lambda = \{ \theta_1, \dots, \theta_{i-1}, b^i, b^{i+1}, \dots,b^{k^*-1} ,\theta_{k^*}, \dots, \theta_L \}, \text{ where  }  |\mu(b^i) -\tau| < |\mu(\theta_{k^*}) - \tau|.
\end{equation}
Using (\ref{eq:bad_para_3}) to solve the constraints (\ref{eq:constr_1}) in $i$ for $C(\theta),$
\noindent
\begin{align}\label{eq:constraint3}
\inf_{b : \mu( \theta_{k^*} )\geq \mu(b) \geq 2\tau-\mu(\theta_{k^*})} C_i I(\theta_i, b^i) + \dots +
C_{k^*-1} I(\theta_{k^*-1}, b^{k^*-1}) \geq 1,\forall i < k^*.
\end{align}
Note that the constraint for each $i<k^*$ in (\ref{eq:constraint3}) has a common term of arm $k^*-1$. Also note that since $\theta_{k^*-1} < 2\tau-\theta_{k^*}$, and $I(a,b+u)$ is monotone increasing in $u$ for $a<b$ and $u>0$,
    $\inf_{\theta_{k^*}\geq b^{k^*-1} \geq 2\tau-\theta_{k^*}} I(\theta_{k'-1}, b^{k'-1}) = I(\theta_{k'-1}, \tau)$.
The optimal solution for (\ref{eq:gl_opt_1}) is:
\begin{equation}\label{eq:sol1_3}
C_{k^*-1} I(\theta_{k^*-1}, 2\tau-\theta_{k^*}) = 1 , \quad C_i = 0 \quad \forall i \setminus \{k^* -1 ,k^*\}.
\end{equation}
Similarly, solving when $\mu_{k*} < \tau$, we get \begin{equation}\label{eq:sol2_3}
    C_{k^*+1}I(\theta_{k^*+1},2\tau - \theta_{k^*})=1, \quad C_i=0 \quad \forall i\setminus\{k^*+1,k^*\}.
\end{equation} Finally, the optimization problem using (\ref{eq:sol1_3}) and (\ref{eq:sol2_3}) yields the lower bound:
\[
\liminf_{T \to \infty} \frac{R^{\pi}(T)}{\log T} \geq C(\theta) = \frac{|\mu_{k^*} - \mu_{k}|}{I(\theta_{k}, 2\tau-\theta_{k^*})}, \text{where } k=k^*-1 \text{ if } \mu_{k^*} \geq \tau \text{ else } k=k^*+1.
\]
This proves the required.
\end{proof}
\vspace{-10 pt}
For variant (iii), the regret bound for identifying the arm with mean closest to the threshold \( \tau \) relies on a single term. When the optimal arm’s mean is \( \geq \tau \) (or \( < \tau \)), the suboptimal arm contributing to the regret is the one immediately below (or above) \( \tau \). The key challenge lies in distinguishing both the optimal and suboptimal arms from \( \tau \), which shapes the regret bound. With these derived lower bounds, we use their insights to build algorithms which approach these bounds.
\section{Algorithms}
In this section, we present three novel algorithms--Threshold-Optimal Sampling in Monotonic Bandits (TOSMB), Ranked-Threshold Optimal Sampling in Monotonic Bandits (RTOSMB), and Proximity-Focused Optimal Sampling in Monotonic Bandits (POSMB) designed to address the three threshold-based objectives within a monotonic structured multi-armed bandit framework, where the true means of the arms are ordered in increasing order, achieving regret bounds that align with the asymptotic lower bounds derived in Theorems \ref{thm:1}, \ref{thm:2}, and \ref{thm:3} for Bernoulli rewards. 

To achieve optimal sampling, each algorithm assigns KL-divergence upper confidence bounds, following the KL-UCB framework \cite{garivier2013klucbalgorithmboundedstochastic}, to define an optimistic index for each arm. For RTOSMB and POSMB, KL-divergence lower confidence bounds (KL-LCB) are used in conjunction with KL-UCB, utilizing the ordered structure of the arms and comparisons with the threshold \(\tau\) to guide optimal arm selection. Across all algorithms, the indices are defined as follows:
\vspace{3 pt}
\begin{align*}
    \text{UI}_k(n) &= \sup \left\{ q : t_k(n) D(\tilde{\mu}_k(n) \| q) \leq \log(n) + c \log(\log(n)) \right\}, \\
    \text{LI}_k(n) &= \inf \left\{ q : t_k(n) D(\tilde{\mu}_k(n) \| q) \leq \log(n) + c \log(\log(n)) \right\},
\end{align*}
where \(t_k(n)\) is the number of times arm \(k\) has been sampled, \(\tilde{\mu}_k(n)= \frac{\sum_{1\leq n\leq T} \mathds{1}\{ k(n) = k \}X_k(n)}{\sum_{1\leq n\leq T} \mathds{1}\{ k(n) = k \}}  \ \) for \(t_k(n)>0\) else \(\tilde{\mu}_k(n)=0\) is the empirical mean, and \(c > 3\) is a tuning constant.

Pseudo codes for TOSMB, RTOSMB and POSMB algorithms are given in Algorithms~\ref{alg:TOSMB}, \ref{alg:RTOSMB} and \ref{alg:POSMB}.
\begin{algorithm}
\caption{Threshold-Optimal Sampling in Monotonic Bandits (TOSMB)}\label{alg:TOSMB}
\begin{algorithmic}[1]
\STATE Sample each arm once, Compute $\text{UI}_i$ for all $i \in [K]$
\STATE Set candidate arm \(\quad \quad\text{ca} \leftarrow \min\{\, i : \tilde\mu_i > \tau \,\} \)
\FOR{$t=K+1$ \textbf{to} $T$}
    \STATE Sample arm: ca
    \IF{$\text{ca} > 1$ \textbf{and} $\text{UI}_{\text{ca}-1} > \tau$} \STATE $\text{ca} \leftarrow \text{ca} - 1$
        
    \ELSIF{$\text{ca} < n$ \textbf{and} $\text{UI}_{\text{ca}} < \tau$}
        \STATE $\text{ca} \leftarrow \text{ca} + 1$
    \ENDIF

\ENDFOR
\end{algorithmic}
\end{algorithm}

\vspace{-2 pt}
\begin{algorithm}
\caption{Ranked-Threshold Optimal Sampling in Monotonic Bandits (RTOSMB)}\label{alg:RTOSMB}
\begin{algorithmic}[1]
\STATE Sample each arm once, Compute $\tilde{\mu}_i$ for all $i \in [K]$
\STATE Set the candidate arm: $\text{ca} \leftarrow \max\{{i:\tilde{\mu}_i<\tau}\}$
\FOR{$t = K+1$ \textbf{to} $T$}
    \STATE Compute $\text{UI}_i$, $\text{LI}_i$ for all $i \in [K]$
    \STATE Let neighbor arm: $\text{na} \leftarrow \text{ca} + 1$
    \IF{$\text{LI}_{\text{na}} \geq \tau$ \AND $\text{UI}_{\text{ca}} \leq \tau$}
        \STATE Sample arm: $\text{ca} + l$
    \ELSIF{$\text{UI}_{\text{ca}} > \tau$}
        \STATE Sample arm: $\text{ca}$
    \ELSIF{$\text{LI}_{\text{na}} < \tau$}
        \STATE Sample arm: $\text{na}$
    \ELSE 
        \STATE Sample arm: $\text{ca}$
    \ENDIF
     \IF{$\text{LI}_{\text{ca}} > \tau$ \AND $\text{ca} > 0$}
        \STATE $\text{ca} \leftarrow \text{ca} - 1$
    \ELSIF{$\text{UI}_{\text{na}} < \tau$ \AND $\text{ca} < K-1$}
        \STATE $\text{ca} \leftarrow \text{ca} + 1$
    \ENDIF
\ENDFOR
\end{algorithmic}
\end{algorithm}

Both TOSMB and RTOSMB exploit the monotonic structure of arms using confidence bounds UI and LI, which give a bound to the true means. UI increases for unexplored arms, while LI decreases, enhancing their likelihood of being captured by our algorithm; however, once selected, UI and LI rapidly converge toward the true mean, enabling consistent identification of the optimal arm, thus balancing the required exploration and exploitation. TOSMB adjusts a candidate arm with UI to find the first arm above $\tau$, while RTOSMB uses both UI and LI to confirm arm positions and then target the  $\l^{\th}$ arm relative to $\tau$, achieving precise identification efficiently.

In the last threshold variant, the objective is to efficiently identify the arm whose expected reward is closest to a predefined threshold \(\tau\). We describe our algorithm POSMB, which approaches the regret derived in Theorem~\ref{thm:3}. The algorithm maintains a candidate arm based on proximity to the threshold and, at each round, updates its decision using upper and lower confidence bounds derived from observed rewards.

\begin{algorithm}
\caption{Proximity-Focused Optimal Sampling in Monotonic Bandits (POSMB)}\label{alg:POSMB}.
\begin{algorithmic}[1]
\STATE Sample each arm once, Compute $\tilde{\mu}_i$ for all $i \in [K]$
\STATE Set candidate arm: \( \quad \quad\text{ca} \leftarrow \arg\min_i |\tilde{\mu}_i - \tau| \)
\FOR{$t = K+1$ \TO $T$}
    \STATE Compute $\text{UI}_i, \text{LI}_i$ for all $i \in [K]$
    \STATE Let neighbor arm: $\text{na} \leftarrow$ $\text{ca}+1$
    \IF{ $\tau$ $\in$ both $[\text{LI}_{ca}, \text{UI}_{ca}]$ and $[\text{LI}_{na}, \text{UI}_{na}]$}
        \STATE Sample arm: $\text{ca}$ if $|\tilde\mu_{ca}-\tau|<|\tilde\mu_{na}-\tau|$ else $\text{na}$
    \ELSIF{ $\tau$ $\in$ $[\text{LI}_{ca}, \text{UI}_{ca}]$}
        \STATE Sample arm:  $\text{ca}$
    \ELSIF{ $\tau$ $\in$ $[\text{LI}_{na}, \text{UI}_{na}]$}
        \STATE Sample arm: $\text{na}$
    \ELSE
        \STATE Sample arm: $\text{ca}$ if $|\text{UI}_{ca}-\tau| < |\text{LI}_{na}-\tau|$ else $\text{na}$
    \ENDIF
     \IF{$\text{LI}_{\text{ca}} > \tau$ \AND $\text{ca} > 0$}
        \STATE $\text{ca} \leftarrow \text{ca} - 1$
    \ELSIF{$\text{UI}_{\text{na}} < \tau$ \AND $\text{ca} < K-1$}
        \STATE $\text{ca} \leftarrow \text{ca} + 1$
    \ENDIF
\ENDFOR
\end{algorithmic}
\end{algorithm}

\section{Simulations}
We empirically evaluate the performance of our proposed algorithms on the three threshold-based identification problems introduced in this work: (i) TOSMB, (ii) RTOSMB, and (iii) POSMB. In the absence of established baselines for these problem settings, we focus our analysis on the convergence behavior of the algorithms toward the corresponding lower bounds, assessing their efficiency.

 We randomly take arm means for $K=10$, and then fix a threshold $\tau \in (0,1)$ such that all constraints are satisfied i.e $\mu_1 < \tau$ and $\mu_K \geq \tau$. Each algorithm is evaluated over a time horizon of $T = 10^6$ iterations, a feasible duration providing insight into asymptotic behavior. To estimate expected performance metrics such as cumulative regret,  we conduct Monte Carlo simulations by running each algorithm independently across 30 randomized trials, with each trial initialized using a different seed. The specific arm means and the value of $\l$ for RTOSMB are detailed in the caption of Figure~\ref {fig:1}.


 It can be seen in Figure~\ref{fig:1} that the average regret approaches the theoretical lower bound as $t \to \infty$ in all the cases, with RTOSMB and POSMB showing closer follow-up with the lower bound.
\newpage
\begin{figure}[htbp]
    \centering

    \begin{minipage}[t]{0.45\linewidth}
        \centering
        \includegraphics[width=\linewidth]{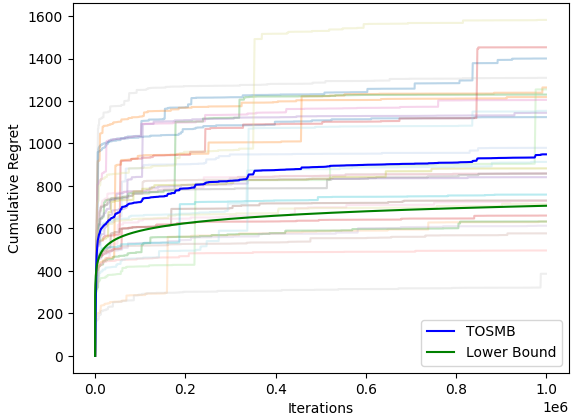}
        \text{(a)}
    \end{minipage}
    \hfill
    \begin{minipage}[t]{0.45\linewidth}
        \centering
        \includegraphics[width=\linewidth]{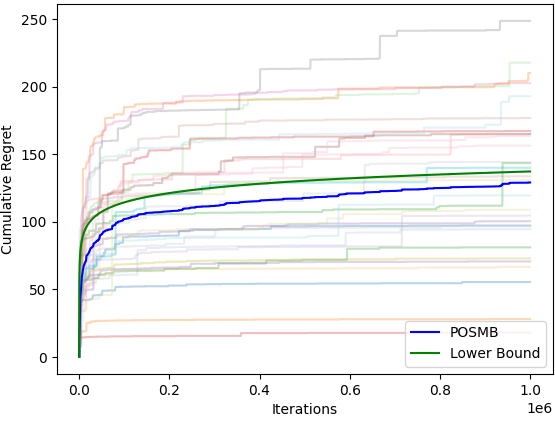}
        \text{(b)}
    \end{minipage}

    \begin{minipage}[t]{0.45\linewidth}
        \centering
        \includegraphics[width=\linewidth]{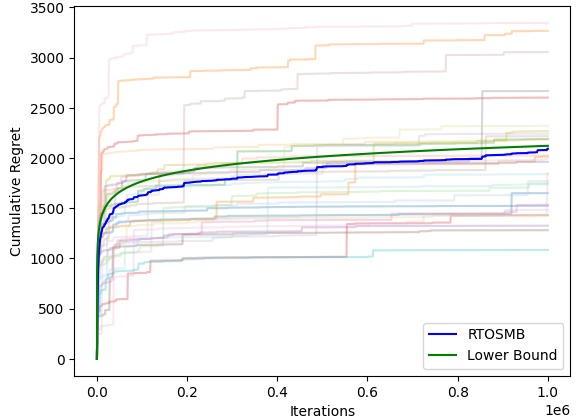}
        \text{(c)}
    \end{minipage}
    \hfill
    \begin{minipage}[t]{0.45\linewidth}
        \centering
        \includegraphics[width=\linewidth]{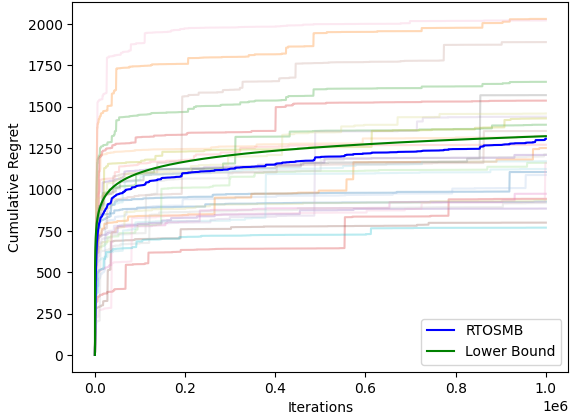}
        \text{(d)}
    \end{minipage}

    \caption{\small (a) shows performance of TOSMB, (b) shows performance of POSMB, (c) shows performance of RTOSMB with $\l=4$ and (d) shows performance of RTOSMB with $\l=-2$ ($3^{rd}$ arm below $\tau$), for $K=10$ arm with true means as $\mu=$\{0.038, 0.041, 0.078, 0.36, 0.533, 0.796, 0.814, 0.85, 0.94, 0.967\} in $10^6$ iterations }
    \label{fig:grid_2x2}\label{fig:1}
\end{figure}

\section{Limitations}
Our framework assumes that the decision maker has a prior knowledge of monotonicity direction (increasing or decreasing). Finding the lower bounds without this prior knowledge can be of interest. Moreover, we have not theoretically established that our proposed algorithms achieve the lower bound. Establishing this is a part of our future work.
\section{Conclusion}
In this paper, we have addressed threshold-based bandit problems under a strict monotonic structure, where arm means satisfy \(\mu_1  < \cdots < \mu_K\) or \(\mu_1  > \cdots > \mu_K\). We proposed efficient algorithms for threshold-crossing identification, \(l\)-th above (below) threshold identification, and threshold-proximity identification, exploiting the structure to achieve optimal decision-making. Our theoretical contributions include asymptotic regret lower bounds that depend only on arms adjacent to $\tau$. Through Monte Carlo simulations, we demonstrated our algorithms empirically approach these bounds. Our work expands traditional bandit theory, providing useful and theoretically justified solutions to applications in clinical dosing, stock selection, communication networks and more.
\bibliography{references}
\end{document}